\documentclass{l4dc2025}

\title[Morphological-Symmetry-Equivariant HGNN]{Morphological-Symmetry-Equivariant Heterogeneous Graph Neural Network for Robotic Dynamics Learning}
\usepackage{times}

\author{%
 \Name{Fengze Xie}$^{\ast 1}$ \Email{fxxie@caltech.edu}\\
 \Name{Sizhe Wei}$^{\ast 2}$ \Email{swei@gatech.edu}\\
 \Name{Yue Song}$^{1}$ \Email{yuesong@caltech.edu}\\
 \Name{Yisong Yue}$^{1}$ \Email{yyue@caltech.edu}\\
 \Name{Lu Gan}$^{2}$ \Email{lgan@gatech.edu}\\
 \addr $\ast$ Indicates equal contribution\\
 \addr $^{1}$ California Institute of Technology\\
 \addr $^{2}$ Georgia Institute of Technology
}

\begin{document}
\maketitle
\footnote{$^\dag$ Project website:~\href{https://lunarlab-gatech.github.io/MorphSym-HGNN/}{https://lunarlab-gatech.github.io/MorphSym-HGNN}. The appendix can be found on the project website.}

\begin{abstract}%
We propose MS-HGNN, a Morphological-Symmetry-Equivariant Heterogeneous Graph Neural Network for robotic dynamics learning, which integrates robotic kinematic structures and morphological symmetries into a unified graph network. By embedding these structural priors as inductive biases, MS-HGNN ensures high generalizability, sample and model efficiency. This architecture is versatile and broadly applicable to various multi-body dynamic systems and dynamics learning tasks. We prove the morphological-symmetry-equivariant property of MS-HGNN and demonstrate its effectiveness across multiple quadruped robot dynamics learning problems using real-world and simulated data. Our code is available at~\href{https://github.com/lunarlab-gatech/MorphSym-HGNN/}{https://github.com/lunarlab-gatech/MorphSym-HGNN/}.$^\dag$

\end{abstract}


\section{Introduction}
A rigid body system is a collection of interconnected components that do not deform under external forces. Existing approaches to planning and controlling for rigid body systems fall into two categories: safe but inflexible methods and adaptive yet risky methods. Traditional methods provide safety and stability by relying on well-understood dynamics models~\citep{10.5555/561828, spong2005robot}, but they struggle in complex, unpredictable environments where modeling becomes difficult. Conversely, machine learning-based approaches offer greater adaptability by learning dynamic interactions and planning strategies across diverse environments~\citep{doi:10.1126/scirobotics.abk2822, 9560769} but suffer from unseen and highly dynamic environments. 

To bridge traditional and learning-based methods, it is essential to incorporate morphological information from the robot’s structure into our learning architecture. By embedding this structural information, the model can implicitly account for the robot’s physical configuration, enhancing interpretability and data efficiency. The morphology of a rigid body system has two key components: the kinematic structure and morphological symmetry. A kinematic structure~\citep{MRUTHYUNJAYA2003279, TAHERI2023105448} consists of interconnected links connected by joints that allow relative motion, such as rotation or translation. Each joint imposes specific movement constraints, enabling the system to perform complex actions through the combination of simpler joint motions. In robotics, kinematic structures are fundamental for modeling and controlling articulated structures like robotic arms, quadrupeds~\citep{7815333, 10161525, yang2023cajun}, and humanoids. Integrating kinematic structure into the learning model can help establish the relative relationships between each component, aligning the model closely with the robot's physical design. On the other hand, morphological symmetries are structural symmetries in a robot’s body that allow it to mimic certain spatial transformations, such as rotations, reflections, or translations~\citep{Smith2023SO2EquivariantDM, ordoñezapraez2024morphologicalsymmetriesrobotics}. Integrating these geometric priors into the learning model can improve data efficiency, enabling the model to generalize better across configurations and tasks. 

In this work, we propose MS-HGNN, a morphological-symmetry-equivariant heterogeneous graph neural network that integrates both the robotic kinematic structures and morphological symmetries into the learning process. We further validate the symmetry properties of the neural network through theoretical analysis, and demonstrate how these embedded features enhance the model's interpretability and efficiency in extensive robotic dynamics learning experiments. 

\section{Related Work}

\textbf{Rigid Body Systems.} 
In robotics, rigid body systems are essential for representing complex articulated structures like robotic arms, quadrupeds, and humanoids. Traditional rigid body modeling relies on established mathematical frameworks
to describe motion and calculate the forces and torques necessary for desired movements~\citep{spong2005robot}. 
On the other hand, data-driven techniques, such as neural networks and reinforcement learning~\citep{9560769, doi:10.1126/scirobotics.abk2822}, have been introduced to model and control rigid body systems, bringing adaptability and flexibility to these systems in diverse or unstructured settings. 
Recently, several approaches have emerged that bridge classic and data-driven methods, leveraging the strengths of both~\citep{O_Connell_2022, doi:10.1177/02783649231169492, pmlr-v211-neary23a, 10611562, 10706036}. These approaches typically embed physical laws as constraints or regulators within the learning model.

\noindent
\textbf{Geometric Deep Learning.} 
Traditional deep learning methods are effective for grid-like data structures, such as images, but often struggle with irregular, non-Euclidean domains. Geometric deep learning overcomes this limitation by developing architectures that preserve the inherent geometric properties of the data~\citep{wang2021incorporating, pmlr-v162-du22e, 9091314, ordoñezapraez2024morphologicalsymmetriesrobotics}. These properties can be captured through various representations~\citep{bronstein2021geometricdeeplearninggrids}, motivating the use of specialized architectures~\citep{satorras2022enequivariantgraphneural, pmlr-v48-cohenc16, wang2021incorporating, NIPS2017_f22e4747}.
In robotics, geometric deep learning offers a physics-informed approach that enhances model interpretability and improves sample efficiency.

\noindent
\textbf{Physics-Informed Learning for Robotics.} Recent advancements in physics-informed learning have shown significant promise in enhancing learning performance by embedding physical laws and dynamics directly into learning models~\citep{10155901, Djeumou2021NeuralNW}. 
Unlike traditional data-driven approaches, physics-informed methods leverage underlying principles~\citep{RAISSI2019686, NEURIPS2019_26cd8eca, cranmer2020lagrangianneuralnetworks, 10598388} to improve model interpretability, robustness, and data efficiency~\citep{SanchezGonzalez2018GraphNA, Kim2021LearningRS, butterfield2024mihgnnmorphologyinformedheterogeneousgraph}.
Our neural network architecture employs geometric deep learning to embed kinematic structure and symmetry information from rigid-body systems, forming a physics-informed model. Compared to prior work~\citep{ordoñezapraez2024morphologicalsymmetriesrobotics, DBLP:journals/corr/abs-2403-17320, mittal2024symmetryconsiderationslearningtask}, our approach not only achieves the same symmetry guarantees but also captures detailed structural information from the kinematic tree through a graph neural network.

\section{Preliminaries}
\subsection{Morphology-Informed Heterogeneous Graph Neural Network}
Heterogeneous Graph Neural Networks (HGNNs)~\citep{Shi2022}, denoted as $\graphG = (\graphV, \graphE)$, are a type of Graph Neural Network (GNN) designed to handle graphs with multiple types of nodes $\graphV$ and edges $\graphE$, capturing complex relationships and rich semantic information. Unlike traditional GNNs, which assume a uniform graph structure, HGNNs apply specialized aggregation and transformation functions tailored to different node and edge types. This makes them particularly effective in applications such as recommendation systems, knowledge graphs, and robotics, where diverse interactions between entities must be accurately modeled. A Morphology-Informed Heterogeneous Graph Neural Network (MI-HGNN)~\citep{butterfield2024mihgnnmorphologyinformedheterogeneousgraph} is an HGNN with node and edge types directly derived from the system’s kinematic structure. Based on the functional roles of nodes within the kinematic structure, we assign them to distinct node classes $\graphV=\{\graphV_1, \dots, \graphV_n\}$, where each class $\graphV_i=\{v_i^1, \dots, v_i^m\}$ contains individual nodes $v_i^j$. Links in the kinematic structure are represented as edges in the graph, with the edge type $e(v_i,v_j)\in\graphE_{ij}$ depending on the node types at both ends, where $\graphE_{ij}\in\graphE$. For example, in a floating-base system, components like the base, joints, and feet can be represented by distinct types of nodes $\graphV_b$, $\graphV_t$, and $\graphV_f$, shown in Fig.\ref{fig:overview-flowchart}(c), while the links connecting these components are modeled as edges. 

\subsection{Morphological Symmetries in Rigid Body Systems}
A rigid body system is a collection of solid bodies that maintain a fixed shape and size while moving under external forces and torques. When connected through joints allowing relative motion, these bodies form a kinematic tree that defines the structured movement of interconnected rigid bodies. Morphological symmetry arises from replicated kinematic chains and body parts with symmetric mass distributions. The morphological symmetry group $\groupG$ represents feasible state transformation including reflection and rotation that map a robot state $(\vq, \dot{\vq})$ to an equivalent reachable state $(g \morphOp \vq, g \morphOp \dot{\vq})$, where $\vq \in \mathbb{R}^{n_q}$ is the generalized position coordinates with $n_q$ as the number of states. We distinguish between standard group actions ($g \triangleright x$), conjugate actions on linear maps ($g \diamond A$), and morphological symmetry-induced state mappings ($g \morphOp x$), following the conventions in~\cite{ordoñezapraez2024morphologicalsymmetriesrobotics}. The formal definition of morphological symmetry action is given by:
\begin{align}\label{ms_def}
    (g \morphOp \vq, g \morphOp \dot{\vq}) := \left(\begin{bmatrix}
        \mX_g \mX_B \mX_g^{-1}\\
        \rho_{\spaceM}(g)\vq_{js}
    \end{bmatrix}, \begin{bmatrix}
        \mX_g \dot{\mX}_B \mX_g^{-1}\\
        \rho_{\tangent_\vq\spaceM}(g)\dot{\vq}_{js}
    \end{bmatrix} \right).
\end{align}
Here, the base transformation \mbox{$g \morphOp \mX_B = \mX_g \mX_B \mX_g^{-1} \in \mathbb{SE}_d$}, where \mbox{$\mX_B := \begin{bmatrix}
    \mathbf{R}_B & \mathbf{r}_B\\
    \mathbf{0} & 1
\end{bmatrix} \in \mathbb{SE}_d$} represents the base configuration, with \mbox{$\mathbf{R}_B \in \mathbb{SO}_d$} as the rotation matrix representing the base's orientation. The transformation matrix \mbox{$\mX_g := \begin{bmatrix}
    \mathbf{R}_g & \mathbf{r}_g\\
    \mathbf{0} & 1
\end{bmatrix} \in \mathbb{E}_d$} includes \mbox{$\mathbf{R}_g \in \mathbb{O}_d$}, an orthogonal matrix representing reflection or rotation. The joint space transformation is defined as \mbox{$g \morphOp \vq_{js} := \rho_\spaceM(g)\vq_{js}$}, where $\rho_\spaceM(g)$ is the representation of the group element $g$. 
Once the base is identified, the robot's unique kinematic branches are represented by \mbox{$\groupS = \{\groupS_1, \dots, \groupS_k\}$}. Each branch \mbox{$\groupS_i$ has $\ndof(\groupS_i) \in \mathbb{N}$} degrees of freedom and is replicated \mbox{$\nrep(\groupS_i) \in \mathbb{N}$} times throughout the kinematic structure. The instance labels for branch \mbox{$\groupS_i$ are $\groupS_i = \{\groupS_{i,1}, \dots, \groupS_{i, \nrep(\groupS_i)}\}$}. The details of morphological symmetry in kinematic structures are provided in the Appendix.\ref{appendix:kinematic_chain}. This morphological symmetry can be naturally modeled using a GNN, where the kinematic structure is encoded by an adjacency matrix and components like the base, joints, and feet are distinguished by node types. The symmetry is embedded in the GNN architecture itself, enabling weight sharing across repeated structures and improving generalization from proprioceptive sensor inputs.
\begin{figure}[tp]
    \centering
    \vspace{-2mm}
    \includegraphics[width=0.95\textwidth]{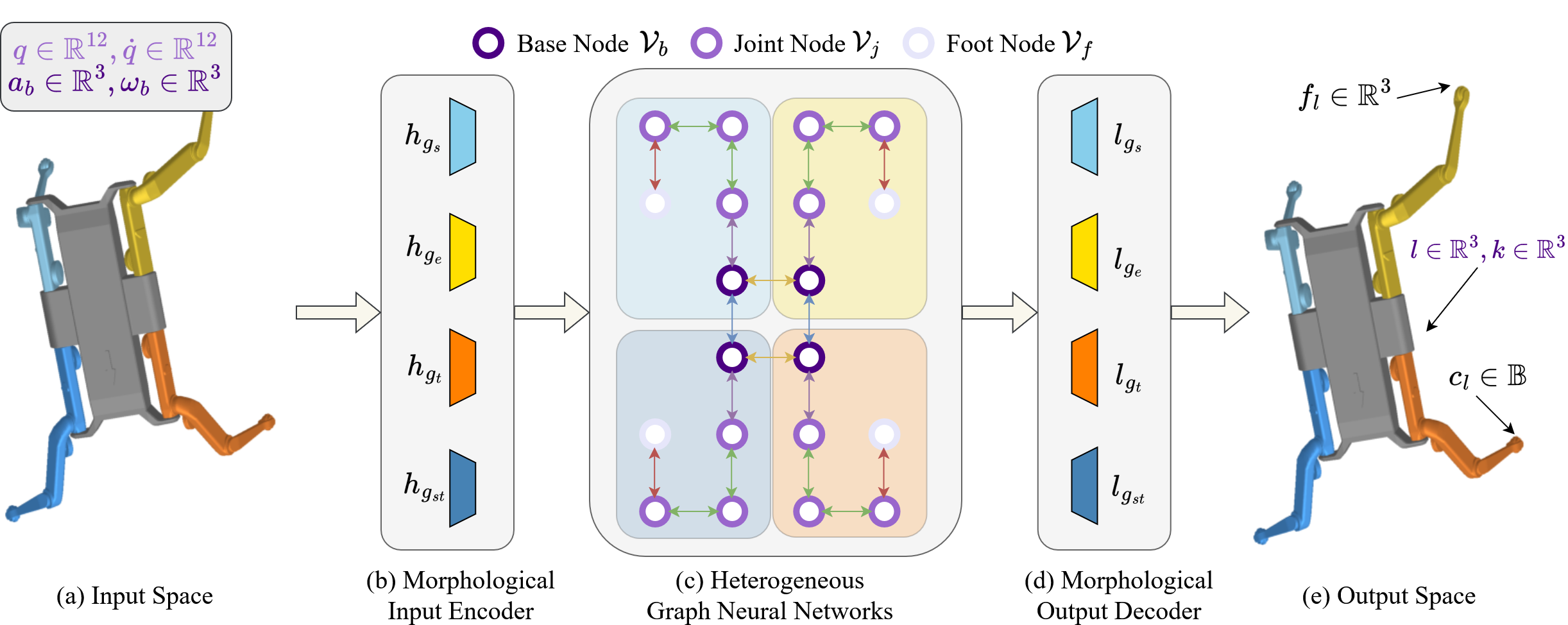} %
    \caption{Overview of the MS-HGNN framework for robots with symmetry type $\groupG := \textcolor[HTML]{6D016B}{\groupK_4}$. (a) The input space consists of the robot's current state observations, which are mapped to corresponding nodes in the HGNN. (b) and (d) The morphological symmetry encoder-decoder pair ensures that the learned representations respect the robot’s morphology. (c) The HGNN is automatically constructed to preserve geometric symmetry. (e) The output space consists of dynamics-relevant variables, obtained from their corresponding nodes in the HGNN.}
    \vspace{-4mm}
    \label{fig:overview-flowchart}
\end{figure}
\begin{figure}[!htp]
    \centering
    \includegraphics[width=\textwidth]{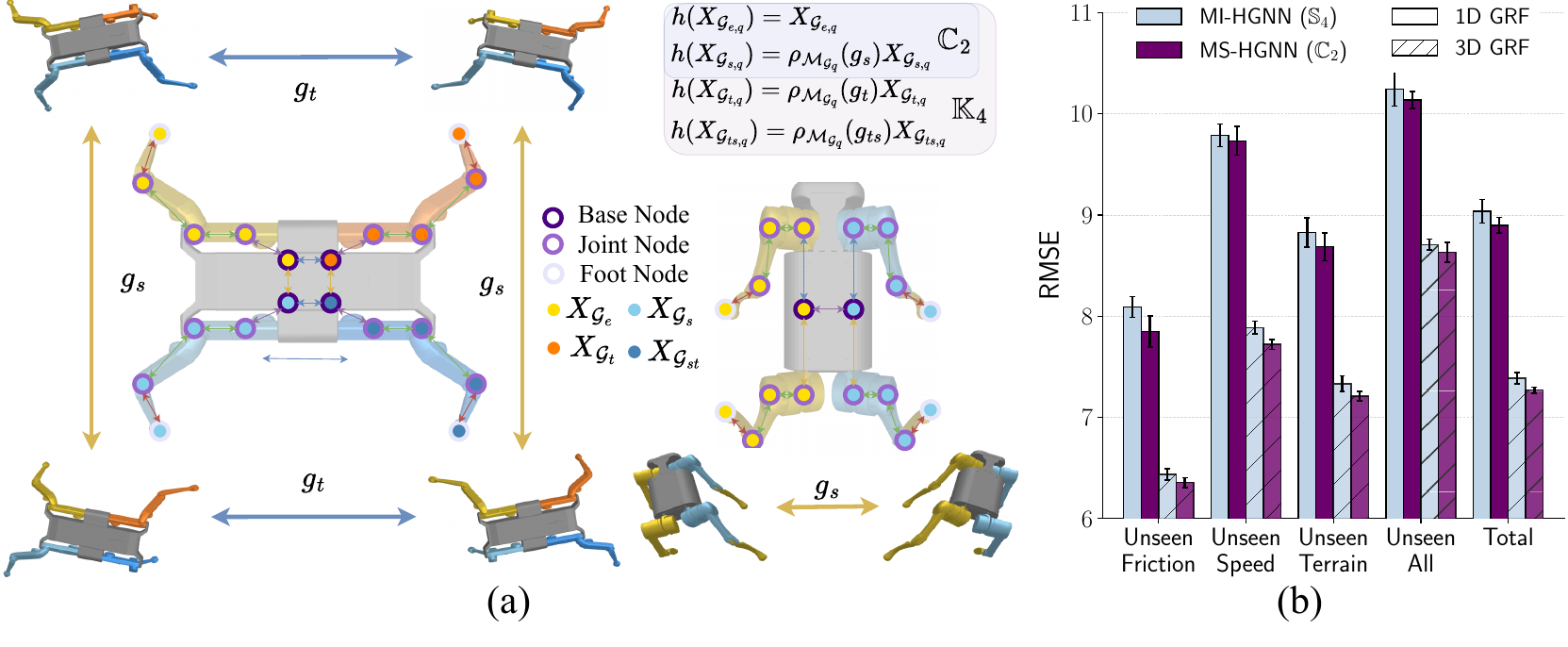} %
    \vspace{-7mm} %
    \caption{(a) The visualization of the MS-HGNN architecture is shown for the morphological symmetry groups $\groupG := \textcolor[HTML]{6D016B}{\groupK_4}$ (left, Solo) and $\groupG := \textcolor[HTML]{8C96C6}{\groupC_2}$ (right, A1). Inputs and outputs of the MS-HGNN are distributed over graph nodes mapped to corresponding robot components (base, joint, foot). Variables representing the entire robot are placed on base nodes. Different node types represent distinct components of the robot’s kinematic structure, with contour colors indicating node types and filling colors denoting the group elements. The encoder and decoder types depend on the group elements, while edges between nodes are determined by node types and the robot’s symmetry. (b) Ground reaction force estimation test RMSE on simulated A1 dataset~\citep{butterfield2024mihgnnmorphologyinformedheterogeneousgraph}.}
    \vspace{-4mm}
    \label{fig:overview}
\end{figure}
\section{Methodology}
This work leverages an HGNN to model the morphological symmetry and kinematic structure of rigid body systems. Our approach consists of two key components as shown in Fig.~\ref{fig:overview-flowchart}: (1) an automatically constructed HGNN that preserves the system's inherent geometric symmetry, and (2) an encoder-decoder module that transforms geometric symmetry into morphological symmetry, ensuring consistency with the system's dynamic properties. In the following sections, we provide a step-by-step framework for constructing an MS-HGNN, guided by the system's kinematic structure and morphological symmetry principles.
\begin{mdframed}[backgroundcolor=lightgray, linewidth=0pt, innerleftmargin=0pt, innerrightmargin=5pt]
    \begin{enumerate}\setlength{\itemsep}{1pt}
\item Determine the morphological symmetry group $\groupG_m < \groupG_\groupE$ and the unique kinematic branches $\groupS$ of the system, where $\groupG_\groupE$ is the generalized euclidean group.
\item Create subgraphs for all kinematic branches as \\$\graphG_i = \{\graphG_{i, 1}(\groupS_{i, 1}), \dots, \graphG_{i, \nrep(\groupS)}(\groupS_{i, \nrep(\groupS_i)})\}$, where $\graphG_{i, j_1}\cong\graphG_{i, j_2}, \forall j_1, j_2 \in \mathbb{N} \le \nrep(\groupS_i)$.
\item Label each subgraph $\graphG_{i, j}$ as $\graphG_{p, q}$, where $p \le |\groupG_m|$ corresponds to an element in group $\groupG_m$, and subgraphs with same $q$ lies in the same orbit. 
\item For any subgraph class $\{\graphG_q\}$, including the base node $\{\graphV_b\}$ that lacks the full set of $|\groupG_m|$ graphs, complete each group orbit by replicating elements along missing transformations and label them as $\graphG_{p, q}$.

\item Connect $\{\graphV_{b, p}\}$ with Cayley Graph~\citep{cayley1878desiderata}. Connect each subgraph $\graphG_{p, q}$ to $\graphV_{b, p}$ with edge type $\graphE_q$, formalizing a graph $\graphG$.
\item Add input encoders and output decoders for each node based on the subgraph class $p$ it belongs to, ensuring morphological symmetry equivariance $\groupG_m$ in our GNN.
\end{enumerate} 
\end{mdframed}

Next, we provide a mathematical proof demonstrating that our constructed graph is equivariant under morphological symmetry transformations. Details of the proof are provided in the Appendix.\ref{appdendix:proof}. After completing the first five steps of our construction process, we obtain a graph $\graphG$ that preserves the system’s inherent geometric symmetry and is composed of subgraphs $\{\graphG_1, \dots, \graphG_q\}$. Each subgraph $\graphG_i$ is further subdivided into instances \mbox{$\{\graphG_{i,1}, \dots, \graphG_{i,p}\}$}, where $p \in \mathbb{N}$ denotes the number of instances and $\nftr(\graphG_i) \in \mathbb{N}$ represents the number of node features per instance. The parameter $q$ corresponds to the types of kinematic chains (e.g., legs, arms), while $p$ identifies the type of element within a group. To ensure that the learned representations respect the morphological symmetry group $\groupG$, we integrate an additional encoder-decoder pair, as shown in Fig.~\ref{fig:overview-flowchart}(b), (d). This enables the HGNN to capture structural equivalences, preserving the morphological symmetry of the robot within the overall learning framework.

We define two types of group actions: the Euclidean reflection and rotation group action, denoted as $g_m \Glact (\cdot)$, and the morphological reflection and rotation group action, denoted as $g_m \morphOp (\cdot)$. 
For each subgraph instance $\graphG_{p,q}$, the Euclidean group action on our graph satisfies the property 
$g_m \Glact \graphG_{p,q} = \graphG_{g_m(p),q},$
where $g_m$ is an element of the morphological transformation group $\groupG_m$. We further define $\rho_{\graphG_q}(g_m) \in \mathbb{R}^{p \times p}$ as the permutation matrix associated with the group action $g_m$.

Consequently, the group action on a stack of subgraph instances can be expressed as:
\begin{equation}
    \rho_{\graphG_q}(g_m)
    \begin{bmatrix}
        \graphG_{p_1,q} &
        \graphG_{p_2,q} &
        \cdots
    \end{bmatrix}^T = 
    \begin{bmatrix}
        \graphG_{g_m(p_1),q} &
        \graphG_{g_m(p_2),q} &
        \cdots
    \end{bmatrix}^T.
\end{equation}

We denote node space representation as an identity matrix as $\rho_{b\spaceM_{\graphG_q}}(g_m) := I_{\nftr(\groupG_m)}$. The graph space permutation matrix $g_m \Glact X_{\graphG} =  \rho_bX_{\graphG}$ is defined as
\begin{equation}
    \rho_b := \operatorname{diag} \big( \rho_{b\spaceM_{[\graphG_1]}}(g_m), \dots, \rho_{b\spaceM_{[\graphG_k]}}(g_m) \big),
    \quad \text{with} \quad
    \rho_{b\spaceM_{[\graphG_i]}}(g_m) := \rho_{\graphG_i}(g_m) \otimes \rho_{b\spaceM_{\graphG_i}}(g_m).
    \label{eq:graph_space_permutation}
\end{equation}
\begin{theorem}[Permutation Automorphism]\label{thm:permutation_auto}
    Assume our $\graphG$ with adjacency matrix $A_\graphG$ and node features $X_\graphG$, where different types of edges and nodes are represented by different integers. The mapping $\phi_{\rho_b}: \graphG\rightarrow\graphG$ is an automorphism if the edge and node features are preserved as: 
    \begin{align}
        \forall \rho_b \in \groupG_m, \quad \phi_{\rho_b}(A_{\graphG}) = \rho_bA_\graphG\rho_b^T = A_{\graphG} \quad \text{and} \quad \phi_{\rho_b}(X_\graphG) = \rho_bX_\graphG = X_{\graphG}
    \end{align}
\end{theorem}
With the above automorphism, the equivariance to Euclidean symmetry immediately follows:
\begin{lemma}[Euclidean Group Equivariance]
    If $\phi_{\rho_b}: \graphG\rightarrow\graphG$ is an automorphism of graph $\graphG$ and $z_{\graphG}$ is the representation of the GNN based on $\graphG$, the GNN is equivariant to Euclidean group actions~\citep{hamilton2020graph}:
    \begin{align}
    \forall g_m \in \groupG_m,\quad    g_m \Glact z_{\graphG}(X_\graphG) = z_{\graphG}\big(\phi_{\rho_b}(X_\graphG)\big) = z_{\graphG}(\rho_b X_\graphG) = z_{\graphG}(g_m \Glact X_\graphG).
    \end{align}
\end{lemma}
However, we would like our neural network to achieve equivariance on morphological reflection and rotation transformation groups, which requires $\forall g_m \in \groupG_m,\ g_m \morphOp z_{\graphG}(X_\graphG) = z_{\graphG}(g_m \morphOp X_\graphG)$, rather than Euclidean reflection and rotation group actions.
\begin{theorem}[Morphological-Symmetry-Equivariant HGNN]\label{theorem:msg-gnn}
With the input encoder $h$ and the output decoder $l$ that satisfies the following condition:
    \begin{align}
        \forall g_{m, p} \in \groupG_m,\quad h(X_{\graphG_{p, q}}) = \rho_{\spaceM_{\graphG_q}}(g_{m, p})X_{\graphG_{p, q}}\quad \text{and} \quad l(X_{\graphG_{p, q}}) = \rho_{\spaceM_{\graphG_q}}(g_{m, p})^{-1} X_{\graphG_{p, q}},
    \end{align}
    where $\rho_{\spaceM_{\graphG_q}}$ denotes the transformation of the coordinate frames attached to each joint belonging to the subgraph class $\graphG_q$. $h$ and $l$ transform Euclidean and Morphological symmetries as follows:
    \begin{align}
        \forall g_m \in \groupG_m,\quad g_m \morphOp l(x) = l(g_m \Glact x) \quad \text{and} \quad 
        g_m \Glact h(x) = h(g_m \morphOp x) 
    \end{align}    
    Our GNN is equivariant to morphological group actions:
    \begin{align}
        \forall g_m \in \groupG_m,\quad g_m \morphOp f_{\graphG}(X_\graphG) = f_{\graphG}(g_m \morphOp X_\graphG).
    \end{align}
    where $f_\graphG$ denotes the graph representation $f_\graphG(X_\graphG) = l(z_\graphG(h(X_\graphG)))$.
\end{theorem}

Our proposed MS-HGNN architecture is designed to be equivariant to morphological symmetry and is generalizable to various robotic systems and different kinds of tasks. To demonstrate its effectiveness, we specifically implement the architecture for the Mini-Cheetah and Solo robots, which exhibit the $\groupK_4$ symmetry group, and the A1 robot, which exhibits the $\groupC_2$ symmetry group. These cases were chosen due to the availability of experimental data and their suitability for visualization, as illustrated in Fig.~\ref{fig:overview}.

It is important to note that since both $\groupK_4$ and $\groupC_2$ have elements that are involutions, the encoder and decoder operations are structurally identical. However, this equivalence does not hold for higher-order cyclic symmetry groups such as $\groupC_n$ with $n > 2$, which are commonly found in other symmetric rigid-body robotic systems, such as multi-arm robots.

\section{Experiments}
We evaluate MS-HGNN as a generalizable model across various rigid body tasks, with a focusing on quadruped robots. Our experiments empirically demonstrate that the specialized graph structure of MS-HGNN effectively captures morphological information, validated on multiple tasks: contact state detection using real-world data (classification), Ground Reaction Force (GRF) estimation and centroidal momentum estimation using simulated data (regression) from various quadruped platforms. These components are critical for modeling quadruped dynamics and enabling effective control. Given the generalized velocity, acceleration, and torques of the system as $\dot{\vq}$, $\ddot{\vq}$, and $\tau$, respectively, the dynamics of quadrupeds are governed by:
\begin{align}
    \mM(\vq)\ddot{\vq} + \mC(\vq, \dot{\vq})\dot{\vq} + g(\vq) = \mS^T\tau + \mJ_\text{ext}(\vq)^T\vf_\text{ext},
\end{align}
where $\mM(\vq)$ is the inertia matrix, $\mC(\vq, \dot{\vq})$ is the Coriolis matrix, $g(\vq)$ is the gravity vector, $\mS^T$ is the selection matrix, $\vf_\text{ext}$ represents external forces, and $\mJ_\text{ext}(\vq)$ is the external force contact Jacobian. In quadrupeds, external forces are dominated by GRFs, leading to the approximation $\mJ_\text{ext}(\vq)^T\vf_\text{ext} \approx \sum_{l=1}^4 \mJ_l(\vq)^T\vf_l$, where $\vf_l$ and $\mJ_l(\vq)$ are the GRF and its Jacobian for leg $l$. Accurate GRF estimation and contact state detection are vital for capturing the system's dynamics and are foundational for downstream control and planning~\cite{an2023artificial, arena2022ground}. In addition, centroidal momenta-comprising the linear and angular momentum of the robot's center of mass (CoM) relative to an inertial frame-summarize whole-body motion~\citep{lasa2010feature-base} and are particularly useful involving external disturbances. Accurately estimating centroidal momenta enables adaptive control under dynamic conditions, such as strong wind or human interactions. We compare our method against CNN~\citep{DBLP:conf/corl/LinZYG21}, state-of-the-art $\mathbb{G}$-equivariant models CNN-Aug and ECNN with $\groupC_2$ symmetry~\citep{DBLP:conf/rss/ApraezMAM23}, and the morphology-aware model, MI-HGNN~\citep{butterfield2024mihgnnmorphologyinformedheterogeneousgraph}.
To comprehensively evaluate model performance and generalizability, we select three datasets spanning different tasks and platforms: contact states on Mini-Cheetah~\citep{DBLP:conf/corl/LinZYG21}, GRF estimation on A1~\citep{butterfield2024mihgnnmorphologyinformedheterogeneousgraph}, and centroid momenta estimation on Solo~\citep{ordoñezapraez2024morphologicalsymmetriesrobotics}. This setup allows us to assess MS-HGNN across diverse robotic systems and dynamic learning challenges.

\subsection{Contact State Detection for Mini-Cheetah Robot (Classification)}
This task involves predicting the 4-leg contact state of a quadruped robot from its proprioceptive sensor data. We adapt the real-world dataset from~\cite{DBLP:conf/corl/LinZYG21}, collected on a Mini-Cheetah robot~\citep{katz2019mini} using various gaits across diverse terrains (e.g., sidewalk, asphalt, concrete, pebbles, forest, grass, etc.). The dataset consists of \emph{measured} joint angles \mbox{$\vq \in \mathbb{R}^{12}$}, joint angular velocities \mbox{$\dot{\vq} \in \mathbb{R}^{12}$}, base linear acceleration \mbox{$\va_b \in \mathbb{R}^{3}$}, base angular velocity \mbox{$\boldsymbol{\omega}_b \in \mathbb{R}^{3}$} from an Inertial Measurement Unit (IMU), and \emph{estimated} foot positions \mbox{$\vp_l \in \mathbb{R}^{3}$} and velocities \mbox{$\vv_l \in \mathbb{R}^{3}$} via forward kinematics, where \mbox{$l = \{LF, LH, RF, RH\}$} is the index of each leg. The \emph{ground-truth} binary contact state \mbox{$\vc_l \in \mathbb{B}$}, \mbox{$\mathbb{B}=\{0, 1\}$} is obtained offline using a non-causal algorithm~\citep{DBLP:conf/corl/LinZYG21}. The dataset has around 1 million samples synchronized at 1000 Hz. Following the data split from~\cite{DBLP:conf/rss/ApraezMAM23}, we use the same unseen sequences for testing and allocate $85\%$ and $15\%$ of the remaining data for training and validation. As test sequences contain unseen gait/terrain combinations, this setup is also helpful for evaluating model's generalization ability on out-of-distribution data. 

Given that Mini-Cheetah exhibits \mbox{$\mathbb{G}=\groupK_4$} symmetry, we evaluate models using both $\groupK_4$ group and its subgroup $\groupC_2$. The input comprises 150-step histories of \mbox{$[\vq , \dot{\vq}, \va_b, \vw_b, \vp, \vv] \in \mathbb{R}^{54}$} up to a time step $t$, and the prediction is a 16-state contact state for all legs \mbox{$\hat{\vc} \in \mathbb{B}^4$} at $t$. For graph-based models (MI-HGNN and our MS-HGNN), the input data is grouped into a graph structure and assigned to the corresponding node, i.e., base ($\va_b, \boldsymbol{\omega}_b $), joint ($q_j, \dot{q}_j $) being $j$ the joint index, and foot ($\vp_l, \vv_l$) measurements are fed into the corresponding base ($\mathcal{V}_b$), joint ($\mathcal{V}_t$) and foot ($\mathcal{V}_f$) nodes, respectively, according to their indices. The foot-wise contact state prediction is output from the corresponding foot node ($\mathcal{V}_f$). For our MS-HGNN-$\mathbb{C}_2$ and MS-HGNN-$\mathbb{K}_4$ models, all base nodes receive identical inputs. We employ 8 message-passing layers with a hidden size of 128 and train the network for up to 49 epochs using a learning rate of $10^{-4}$. We evaluate all models using the metrics form~\cite{DBLP:conf/rss/ApraezMAM23, butterfield2024mihgnnmorphologyinformedheterogeneousgraph}, including foot-wise binary F1-score, averaged F1-score across legs, and 16-state contact state accuracy (correct only if all four legs are classified correctly).

\begin{figure}[t!]
\captionsetup{format=plain}
    \centering
    \includegraphics[width=\textwidth]{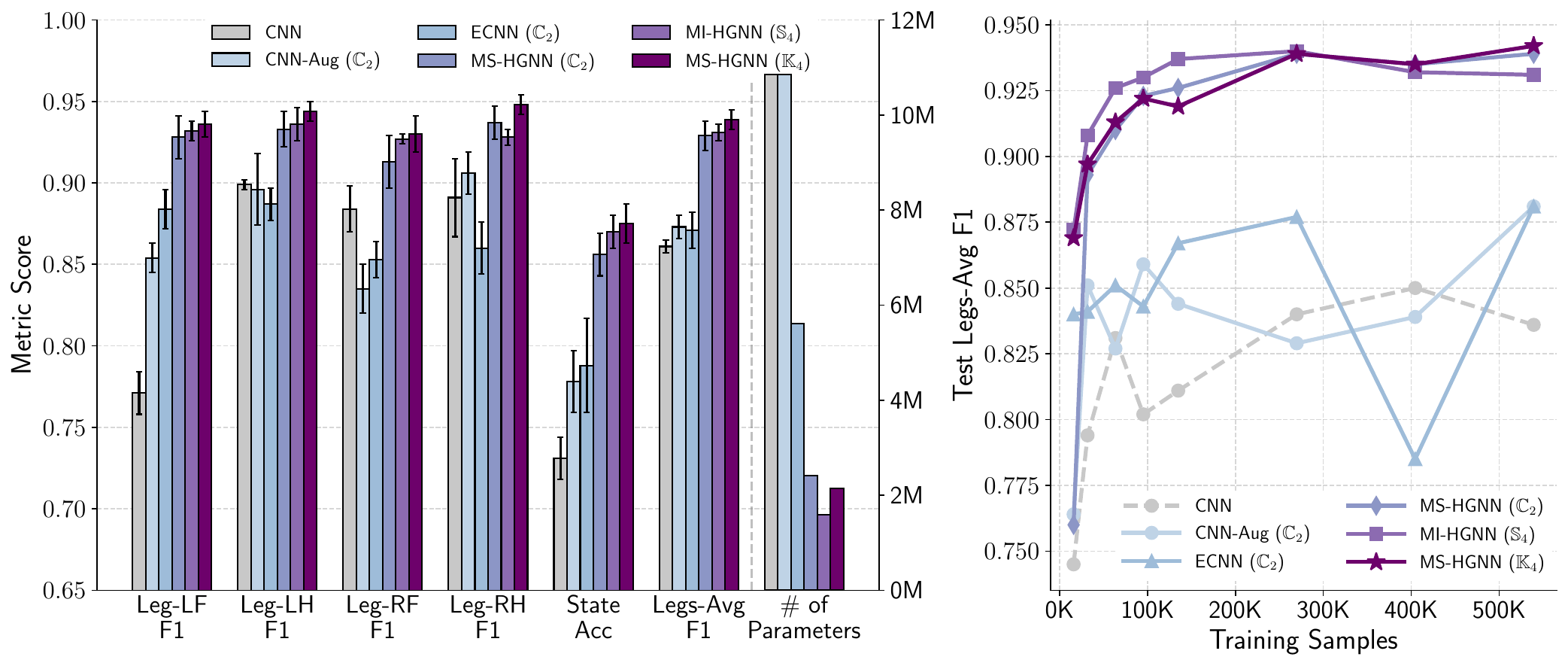}
    \caption{\textbf{Contact state detection results} on the real-world Mini-Cheetah dataset~\citep{DBLP:conf/corl/LinZYG21}. Left: F1 scores per leg, averaged F1 score, and 16-state accuracy (4-run average). Parameter counts are shown for each method. Right: Averaged F1 scores versus training set size. MS-HGNN (\textcolor[HTML]{8C96C6}{$\groupC_2$} \& \textcolor[HTML]{6D016B}{$\groupK_4$}) achieve around 0.9 averaged F1-score with just 5$\%$ of training data.}
    \vspace{-6mm}
    \label{fig:contact-detection-exp}
\end{figure}

The classification results (mean $\pm$ std across 4 runs) and parameter sizes are summarized in Fig.~\ref{fig:contact-detection-exp}-left.$^1$\footnote{$^1$ Detailed numerical results corresponding to all figures are included in Appendix~\ref{appendix:results}.} Compared to non-graph-based models (CNN, CNN-Aug, ECNN), graph-based networks achieve substantially better performance with significantly fewer parameters. Specifically, MS-HGNN-$\mathbb{K}_4$ improved $11\%$ contact state accuracy over ECNN, the best-performing non-graph-based model, using only $38\%$ of ECNN's parameters. This demonstrated the effectiveness and efficiency of using morphological structure via graph representations. The morphology-informed graph network enforces information flow consistent with the robot's kinematic structure, embedding physical priors into message-passing. This architecture utilizes physical knowledge as a \emph{prior} and increases the model's \emph{causality}, and also reduces the parameter count needed to capture complex dynamics of the robot. Among graph-based models, our proposed MS-HGNN-$\mathbb{K}_4$ outperforms MI-HGNN in both averaged F1-score (0.939 vs. 0.931) and accuracy (0.875 vs. 0.870), showing the benefit of preserving morphological-symmetry. In contrast, MI-HGNN follows geometric symmetry $\mathbb{S}_4$, permutation-equivariant for any legs, which over-constrains the model and leads to suboptimal results. Among $\mathbb{C}_2$ models (excluding MI-HGNN), MS-HGNN-$\mathbb{C}_2$ achieves the best performance. Furthermore, the performance gap between $\mathbb{K}_4$ and $\mathbb{C}_2$ variants of MS-HGNN emphasizes the benefits of exploiting the morphological symmetries in robotic systems, which is offering equivalent to doubling the effective dataset size through $\mathbb{K}_4$ compared to $\mathbb{C}_2$. 

\noindent
\textbf{Trainable parameters in $\mathbb{G}$-equivariant networks.} $\mathbb{G}$-equivariant network reduces the trainable parameters by $1/|\mathbb{G}|$ ($|\mathbb{G}|$ is the group order) compared to unconstrained neural network of the same architectural. Thus, ECNN-$\groupC_2$ or EMLP-$\groupC_2$~\citep{DBLP:conf/rss/ApraezMAM23} have approximately twice as many parameters as ECNN-$\mathbb{K}_4$ or EMLP-$\groupK_4$. Interestingly, our MS-HGNN achieves $\mathbb{G}$-equivariance via structured graphs and minimal edge connections instead of $\mathbb{G}$-equivariant layers, yielding comparable parameters for MS-HGNN-$\mathbb{C}_2$ and MS-HGNN-$\mathbb{K}_4$ (see Fig.~\ref{fig:contact-detection-exp}-left).

\noindent
\textbf{Sample efficiency of MS-HGNN.} We assess sample efficiency by varying training set size and reporting the test averaged F1-scores (Fig.~\ref{fig:contact-detection-exp}-right). Like MI-HGNN, both MS-HGNN variants ($\mathbb{C}_2$ and $\mathbb{K}_4$) achieve $\sim$0.9 averaged F1-scores with only 5$\%$ of the training data. This demonstrates the utility of morphological priors in real-world settings, where data is limited and expensive to collect.

\subsection{Ground Reaction Force Estimation for A1 Robot (Regression)}
Estimating ground reaction forces is essential for accurate legged robot dynamics modeling and control. Our graph-based network is well-suited to this task, naturally integrating multi-modal sensory inputs from local frames via message passing. We evaluate on the simulated GRF dataset from~\cite{butterfield2024mihgnnmorphologyinformedheterogeneousgraph}, collected using the \textsc{Quad-SDK} simulator~\citep{norby2022quad} on an A1 quadruped robot exhibiting $\mathbb{G}=\groupC_2$ symmetry. The data is synchronized at 500 Hz and includes joint states (\mbox{$\vq \in \mathbb{R}^{12}$}, \mbox{$\dot{\vq} \in \mathbb{R}^{12}$}, \mbox{$\boldsymbol{\tau} \in \mathbb{R}^{12}$}), base linear acceleration ($\boldsymbol{a}_b \in \mathbb{R}^3$), base angular velocity ($\boldsymbol{\omega}_b \in \mathbb{R}^3$), and ground-truth GRFs ($\boldsymbol{f}_l \in \mathbb{R}^3$, $l$ is the leg index). Following~\cite{butterfield2024mihgnnmorphologyinformedheterogeneousgraph}, we use a 150-step history of input to predict both $Z$-axis (1D) and full 3D GRFs. The hyperparameters and training settings for MS-HGNN and MI-HGNN remain consistent with the previous task. Fig~\ref{fig:overview}(b) presents test the Root Mean Square Error (RMSE) over 4 runs on test sequences with unseen terrain frictions, robot speeds, and terrain types. MS-HGNN-$\mathbb{C}_2$ consistently achieves lower RMSE than MI-HGNN, with an average improvement of $1.62\%$ in 3D and $1.50\%$ in 1D GRF prediction. The relatively modest gain in 3D is attributed to the low GRF magnitudes in the $X$/$Y$ directions in this dataset. Theses results highlight the benefit of preserving morphological symmetry in MS-HGNN over the heuristic design of MI-HGNN for force estimation tasks.

\begin{figure}[!t]
    \centering
    \includegraphics[width=\textwidth]{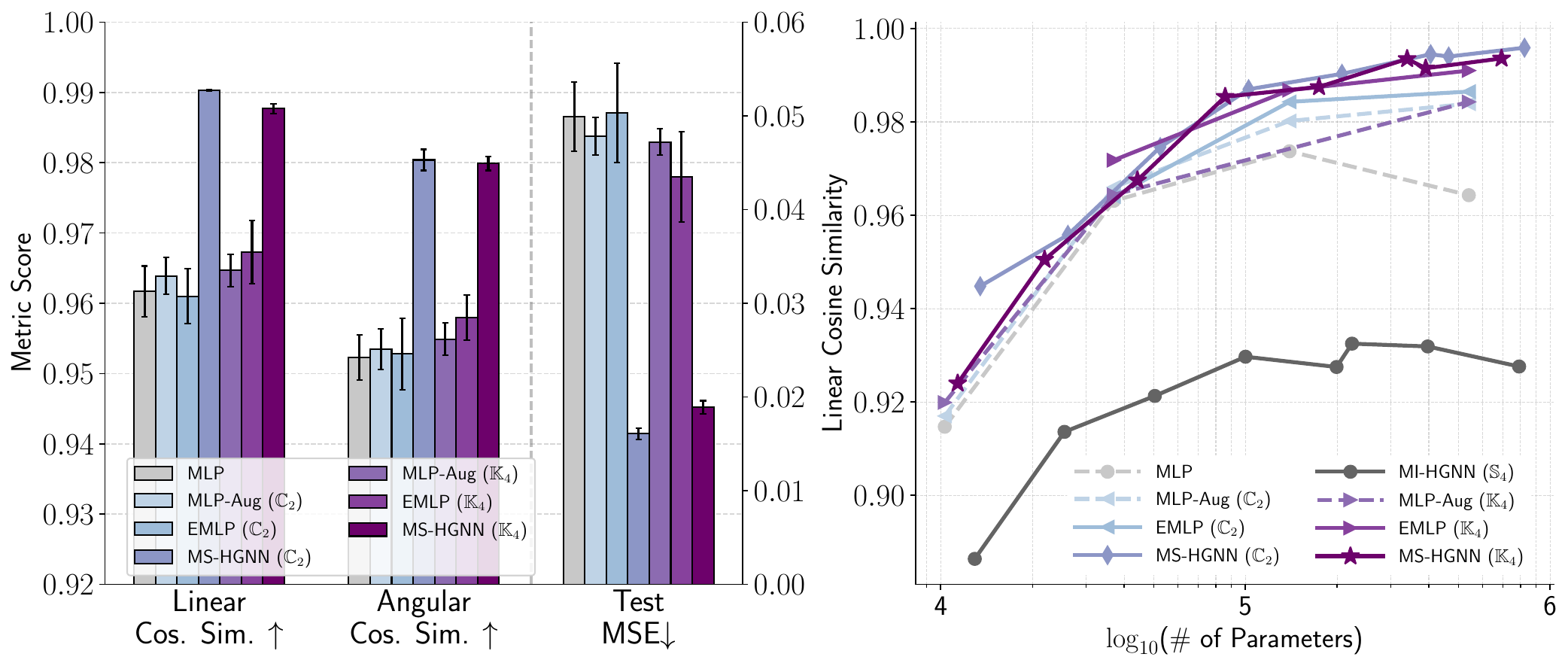}
    \caption{\textbf{Centroidal momentum estimation results} on the synthetic Solo dataset~\citep{DBLP:conf/rss/ApraezMAM23}. Left: Test linear/angular cosine similarity and MSE of predictions, averaged over 4 runs. Right: Linear cosine similarity for models of varying sizes. Our MS-HGNN (\textcolor[HTML]{8C96C6}{$\groupC_2$} and \textcolor[HTML]{6D016B}{$\groupK_4$}) exhibit superior model efficiency without overfitting.}
    \vspace{-4mm}
    \label{fig:com-exp}
\end{figure}

\subsection{Centroidal Momentum Estimation for Solo Robot (Regression)}
In this experiment, we estimate the robot's linear \mbox{$\boldsymbol{l} \in \mathbb{R}^3$} and angular momentum \mbox{$\boldsymbol{k} \in\mathbb{R}^3$} from its joint-space position and velocities \mbox{$\vq \in \mathbb{R}^{12}, \dot{\vq} \in \mathbb{R}^{12}$}. The simulated dataset is generated by~\cite{ordoñezapraez2024morphologicalsymmetriesrobotics} using \textsc{Pinocchio}~\citep{carpentier2019pinocchio} for the Solo quadruped robot exhibiting $\mathbb{G} = \mathbb{K}_4$ symmetry. Different from previous contact estimation problems, this task introduces a new challenge of predicting \emph{angular} momentum from multiple base nodes. To adapt MS-HGNN to this task, we construct models under $\mathbb{C}_2$ and $\mathbb{K}_4$ by attaching the morphology encoder to all joint nodes and morphology decoder to base nodes, and use the MSE losses across base nodes for training. We compare our models ($\mathbb{C}_2$ and $\mathbb{K}_4$) with the unconstrained MLP, MLP-Aug, EMLP~\citep{DBLP:conf/rss/ApraezMAM23}, and MI-HGNN~\citep{butterfield2024mihgnnmorphologyinformedheterogeneousgraph}. Quantitative results (mean and standard deviation over 4 runs) evaluated by cosine similarity and MSE are given in Fig.~\ref{fig:com-exp}-left. MI-HGNN is excluded from the figure due to its incomparable performance (linear cosine similarity 0.9301$\pm$0.0017, angular cosine similarity 0.5173 $\pm$ 0.0016, test MSE 0.3421 $\pm$ 0.0009). Our model outperform all baselines by a significant margin. The degraded performance of MI-HGNN stems from its use of $\groupS_4$ symmetry, which misaligns with the robot's true morphological structure, making it ineffective in learning angular dynamics. In contrast, MS-HGNN embeds $\groupC_2$ and $\groupK_4$ symmetries, enabling effective learning of both linear and angular momentum. We further evaluate the \textbf{model efficiency} of MS-HGNN by varying the numbers of parameters and reporting test linear cosine similarity in Fig.~\ref{fig:com-exp}-right. Notably, MS-HGNN-$\mathbb{C}_2$ achieves 0.9448 cosine similarity with only 13,478 parameters, showing high model efficiency. Furthermore, the performance of both MS-HGNN variants improves steadily as model size increases, whereas MI-HGNN and MLP tend to overfit when scaled up.

\section{Conclusions}
In this work, we introduce MS-HGNN, a general and versatile network architecture for robotic dynamics learning by integrating both robotic kinematic structures and morphological symmetries. Through rigorous theoretical proof and extensive empirical validation, we demonstrated the advantages of using a morphology-informed graph network structure and morphological-symmetry-equivariant property in robotic dynamics learning. Experiments show that MS-HGNN consistently achieves superior performance, generalizability, sample efficiency and model efficiency across a variety of tasks and platforms, making MS-HGNN particularly suitable for data-scarce robotic applications. Furthermore, the modularity of MS-HGNN allows easy adaptation to diverse robotic systems with varying morphological structures. Future work will extend the framework to embed temporary symmetries in robotic systems and deploy to real robot for more challenging tasks.

\acks{Y. Song was supported by gifts from Cisco and OpenAI. We thank D. Butterfield, Z. Gan, and X. Wu for insightful discussions, and L. Zhao for assistance with the experiments.}

\bibliography{reference}
\newpage

\appendix

\section{Morphological Symmetry of Kinematic Chains}\label{appendix:kinematic_chain}

A typical quadruped consists of a single kinematic chain $\groupS_s = \{\groupS_{leg}\}$, which is replicated $\nrep(\groupS_{leg}) = 4$ times. The action of a morphological symmetry in the joint space results in a permutation of the roles of branches with the same type denoted as $g \Glact s_{i, j} := s_{i, g(j)} \in \mathbb{S}_i$ is the label that $j$ is mapped to under the permutation induced by $g$. This leads to the decomposition of the joint space configuration:
\begin{equation}
    \small
 \begin{split}
    g \Glact s_{i,j} &:= s_{i, g(j)} \in \mathbb{S}_i, \quad \text{and} \\
    \rho_{\groupS_i}(g)
    \begin{bmatrix}
            s_{i,1} & s_{i,2} & \cdots
        \end{bmatrix}^T
        & =
        \begin{bmatrix}
            s_{i, g(1)} & s_{i, g(2)} & \cdots
        \end{bmatrix}^T \quad
        | \quad \forall i \in [1, k], j\in[1, \nrep(s_i)],
    \end{split}
    \label{eq:permutation_per_substructures}
\end{equation}
where $\rho_{\groupS_i}(g)$ is the permutation representation acting on the labels of the instances of branch type $s_i$. Following our example with the quadruped robot, the action of $g$ induces a permutation of the left and right configurations $g \Glact s_{leg, lf} = s_{leg, rf}$, $g \Glact s_{leg, rf} = s_{leg, lf}$, $g \Glact s_{leg, lh} = s_{leg, rh}$, and $g \Glact s_{leg, rh} = s_{leg, lh}$. Given that these permutations do not mix the distinct branch types, we can adopt a basis for the joint space configuration space, leading to the decomposition of its associated group representation.
\begin{equation}
    \small
    \begin{split}
        \spaceM &:= \spaceM_{[\groupS_1]} \times \ldots \times \spaceM_{[\groupS_k]} \subseteq \mathbb{R}^{n_j},
        \quad \spaceM_{[\groupS_i]}:= \otimes_{j=1}^{\nrep(\groupS_i)} \spaceM_{\groupS_i} \quad \text{and} \\
        \rho_\spaceM &:= \rho_{\spaceM_{[\groupS_1]}} \oplus \cdots \oplus \rho_{\spaceM_{[\groupS_k]}}, \quad \rho_{\spaceM_{[\groupS_i]}} := \rho_{\groupS_i} \otimes \rho_{\spaceM_{\groupS_i}}, 
    \end{split}
    \label{eq:joint_space_configuration_decomposition}
\end{equation}
where $\spaceM_{\groupS_i} \subseteq\mathbb{R}^{\ndof(\groupS_i)}$ represents the configuration space of a single instance of type $\groupS_i$. For further details, we refer the reader to \cite{ordoñezapraez2024morphologicalsymmetriesrobotics}.

\section{Proof Details}\label{appdendix:proof}
\renewcommand{\thetheorem}{\ref{thm:permutation_auto}} %
\begin{theorem}[Permutation Automorphism]
    Assume our $\graphG$ with adjacency matrix $A_\graphG$ and node features $X_\graphG$, where different types of edges and nodes are represented by different integers. The mapping $\phi_{\rho_b}: \graphG\rightarrow\graphG$ is an automorphism if the edge and node features are preserved as: 
    \begin{align}
       \forall \rho_b \in \groupG_m, \quad \phi_{\rho_b}(A_{\graphG}) = \rho_bA_\graphG\rho_b^T = A_{\graphG} \quad \text{and} \quad \phi_{\rho_b}(X_\graphG) = \rho_bX_\graphG = X_{\graphG}
    \end{align}
\end{theorem}

\begin{proof}
It is easy to find out that the mapping $\phi_{\rho_b}$ satisfies the following properties:
\begin{align}
    \text{Injective:}\quad  \forall \rho_b \in \groupG_m,\quad \text{if} \quad A_{\graphG_1} = A_{\graphG_2}, \quad \phi_{\rho_b}(A_{\graphG_1}) = \phi_{\rho_b}(A_{\graphG_2})\nonumber \\
    \forall \rho_b \in \groupG_m,\quad \text{if} \quad X_{\graphG_1} = X_{\graphG_2},\quad \phi_{\rho_b}(X_{\graphG_1}) = \phi_{\rho_b}(X_{\graphG_2})\\
    \text{Surjective:}\quad \forall \rho_b \in \groupG_m,\quad \phi_{\rho_b}(\phi_{\rho_b}(A_{\graphG})) = \phi_{\rho_b}(A_{\graphG})\quad  \text{and} \quad \phi_{\rho_b}(\phi_{\rho_b}(X_{\graphG})) = \phi_{\rho_b}(X_{\graphG})\\
    \text{Homomorphism:}\quad \forall \rho_b \in \groupG_m,\quad\phi_{\rho_b}(A_{\graphG_1}A_{\graphG_2}) =  \rho_bA_{\graphG_1}(\rho_b^T\rho_b)A_{\graphG_2}\rho_b^T=\phi_{\rho_b}(A_{\graphG_1})\phi_{\rho_b}(A_{\graphG_2})\nonumber\\
    \forall \rho_b \in \groupG_m,\quad\phi_{\rho_b}(X_{\graphG_1}X_{\graphG_2}) = \rho_bX_{\graphG_1}\rho_bX_{\graphG_2}= \phi_{\rho_b}(X_{\graphG_1})\phi_{\rho_b}(X_{\graphG_2})
\end{align}
    Hence $\phi$ is an isomorphism from $\graphG$ to $\graphG$, which is also known as an automorphism.
\end{proof}

\renewcommand{\thetheorem}{\ref{theorem:msg-gnn}} %
\begin{theorem}[Morphological-Symmetry-Equivariant HGNN]
With the input encoder $h$ and the output decoder $l$ that satisfies the following condition:
    \begin{align}
        \forall g_{m, p} \in \groupG_m,\quad h(X_{\graphG_{p, q}}) = \rho_{\spaceM_{\graphG_q}}(g_{m, p})X_{\graphG_{p, q}}\quad \text{and} \quad l(X_{\graphG_{p, q}}) = \rho_{\spaceM_{\graphG_q}}(g_{m, p})^{-1} X_{\graphG_{p, q}},
    \end{align}
    where $\rho_{\spaceM_{\graphG_q}}$ denotes the transformation of the coordinate frames attached to each joint belonging to the subgraph class $\graphG_q$. $h$ and $l$ transform Euclidean and Morphological symmetries as follows:
    \begin{align}
        \forall g_m \in \groupG_m,\quad g_m \morphOp l(x) = l(g_m \Glact x) \quad \text{and} \quad 
        g_m \Glact h(x) = h(g_m \morphOp x) 
    \end{align}    
    Our GNN is equivariant to morphological group actions:
     \begin{align}
        \forall g_m \in \groupG_m,\quad g_m \morphOp f_{\graphG}(X_\graphG) = f_{\graphG}(g_m \morphOp X_\graphG).
    \end{align}
    where $f_\graphG$ denotes the graph representation $f_\graphG(X_\graphG) = l(z_\graphG(h(X_\graphG)))$.
\end{theorem}

\begin{proof}
    With the pre-defined decoder $l$, we can show that the Euclidean group actions can be translated into morphological group actions:
    \begin{align*}
        &l(g_{m, p_2} \Glact X_{\graphG_{p_1, q}}) = \rho_{\spaceM_{\graphG_q}}(g_{m, p_1})^{-1}X_{\graphG_{p_1p_2, q}}\\
        =& \rho_{\spaceM_{\graphG_q}}(g_{m, p_2})\rho_{\spaceM_{\graphG_q}}(g_{m, p_1})^{-1}\rho_{\spaceM_{\graphG_q}}(g_{m, p_2})^{-1}X_{\graphG_{p_1p_2, q}} = g_{m, p_2} \morphOp l(X_{\graphG_{p_1, q}}).
    \end{align*}
    where $g_{m, p_1} \Gcomp g_{m, p_2} = g_{m, p_2} \Gcomp g_{m, p_1}, \forall g_{m, p_1}, g_{m, p_2} \in \groupG_m$. 
    Similarly, for the encoder $h$, the morphological actions can be transformed into Euclidean ones:
    \begin{align*}
        &h(g_{m, p_2} \morphOp X_{\graphG_{p_1, q}}) = \rho_{\spaceM_{\graphG_q}}(g_{m, p_1})\rho_{\spaceM_{\graphG_q}}(g_{m, p_2})X_{\graphG_{p_1p_2, q}}\\
        =& \rho_{\spaceM_{\graphG_q}}(g_{m, p_1 \Gcomp p_2})X_{\graphG_{p_1p_2, q}} = g_{m, p_2} \Glact h(X_{\graphG_{p_1, q}}).
    \end{align*}
    Then for the graph representation $f_\graphG(X_\graphG) = l(z_\graphG(h(X_\graphG)))$, we have
    \begin{align*}
    &f_{\graphG}(g_m \morphOp X_\graphG) = l(z_\graphG(h(g_m \morphOp X_\graphG))) = l(z_\graphG(g_m \Glact h(X_\graphG)))\\
    =& l(g_m \Glact z_\graphG(h(X_\graphG))) = g_m \morphOp l(z_\graphG(h(X_\graphG))) = g_m \morphOp f_{\graphG}(X_\graphG).
\end{align*}
which shows the equivariance property of our MS-HGNN to morphological symmetries.
\end{proof}

\section{Table For Results}
\label{appendix:results}

\vspace{-20pt} 
\begin{table}[!htp]
\centering
\caption{\textbf{Ground reaction force estimation} on the simulated A1 dataset~\citep{butterfield2024mihgnnmorphologyinformedheterogeneousgraph}. This table provides the numerical results corresponding to Fig.~\ref{fig:overview}(b). The metric is the mean$\pm$std of the test RMSE over 4 runs. The best performance is highlighted in \textbf{bold}.}
\label{tab:main-reg-3d}
{\fontsize{9}{11}\selectfont
\begin{tabular}{@{}lcccc@{}}
\toprule
\multirow{2}{*}{\textbf{Test Sequence}} & \multicolumn{2}{c}{\textbf{1D GRF}}                & \multicolumn{2}{c}{\textbf{3D GRF}}                 \\ \cmidrule(l){2-5} 
                                        & \textbf{MI-HGNN}   & \textbf{MS-HGNN (\textcolor[HTML]{8C96C6}{$\mathbb{C}_2$})} & \textbf{MI-HGNN}  & \textbf{MS-HGNN (\textcolor[HTML]{8C96C6}{$\mathbb{C}_2$})} \\ \midrule
Unseen Friction                         & 8.089 $\pm$ 0.102  & \textbf{7.850} $\pm$ 0.154        & 6.437 $\pm$ 0.055 & \textbf{6.355} $\pm$ 0.050        \\
Unseen Speed                            & 9.787 $\pm$ 0.111  & \textbf{9.733} $\pm$ 0.142        & 7.887 $\pm$ 0.064 & \textbf{7.721} $\pm$ 0.048        \\
Unseen Terrain                          & 8.826 $\pm$ 0.144  & \textbf{8.685} $\pm$ 0.136        & 7.332 $\pm$ 0.076 & \textbf{7.208} $\pm$ 0.047        \\
Unseen All                              & 10.245 $\pm$ 0.168 & \textbf{10.137} $\pm$ 0.084       & 8.708 $\pm$ 0.052 & \textbf{8.630} $\pm$ 0.097        \\
\midrule
Total                                   & 9.035 $\pm$ 0.116  & \textbf{8.899} $\pm$ 0.079        & 7.388 $\pm$ 0.056 & \textbf{7.268} $\pm$ 0.032        \\ \bottomrule
\end{tabular}
}
\end{table}

\vspace{-20pt} 
\begin{table}[!htp]
\centering
\caption{\textbf{Contact state detection performance} on the real-world Mini-Cheetah dataset~\citep{DBLP:conf/corl/LinZYG21}. This table reports the numerical results corresponding to Fig.~\ref{fig:contact-detection-exp}-left. Metrics include the mean$\pm$std of F1 score per leg, 16-state accuracy, and the averaged F1 score across 4 runs. \textbf{Bold} and \underline{underlined} values indicate the best and second-best results, respectively.}

\label{tab:main-exp-cls}
{\fontsize{9}{11}\selectfont
\resizebox{\textwidth}{!}{
\begin{tabular}{@{}lccccccc@{}}
\toprule
\textbf{Model (\# of Params)} & \textbf{Sym.}                            & \textbf{Leg-LF F1 $\uparrow$} & \textbf{Leg-LH F1 $\uparrow$}   & \textbf{Leg-RF F1 $\uparrow$} & \textbf{Leg-RH F1 $\uparrow$} & \textbf{State Acc $\uparrow$} & \textbf{Legs-Avg F1 $\uparrow$} \\ \midrule
CNN (10,855,440)            & -                                        & 0.771 $\pm$ 0.013             & 0.899 $\pm$ 0.003               & 0.884 $\pm$ 0.014             & 0.891 $\pm$ 0.024           & 0.731 $\pm$ 0.013             & 0.861 $\pm$ 0.004\\
CNN-Aug (10,855,440)       & $\mathbb{C}_2$                           & 0.854 $\pm$ 0.009             & 0.896 $\pm$ 0.022               & 0.835 $\pm$ 0.015             & 0.906 $\pm$ 0.013        & 0.778 $\pm$ 0.019             & 0.873 $\pm$ 0.007        \\
ECNN    (5,614,770)       & $\mathbb{C}_2$                           & 0.884 $\pm$ 0.012             & 0.887 $\pm$ 0.010               & 0.853 $\pm$ 0.011             & 0.860 $\pm$ 0.016     & 0.788 $\pm$ 0.029             & 0.871 $\pm$ 0.011          \\
MI-HGNN   (1,585,282)     & $\mathbb{S}_4$                           & \underline{0.932} $\pm$ 0.006 & \underline{0.936} $\pm$ 0.010   & \underline{0.927} $\pm$ 0.003 & 0.928 $\pm$ 0.005       & \underline{0.870} $\pm$ 0.010 & \underline{0.931} $\pm$ 0.005       \\ \midrule
MS-HGNN (2,407,810)       & \textcolor[HTML]{8C96C6}{$\mathbb{C}_2$} & 0.928 $\pm$ 0.013             & 0.933 $\pm$ 0.011               & 0.913 $\pm$ 0.016             & \underline{0.937} $\pm$ 0.010 & 0.856 $\pm$ 0.013             & 0.929 $\pm$ 0.009 \\
MS-HGNN  (2,144,642)      & \textcolor[HTML]{6D016B}{$\mathbb{K}_4$} & \textbf{0.936} $\pm$ 0.008    & \textbf{0.944} $\pm$ 0.006      & \textbf{0.930} $\pm$ 0.011    & \textbf{0.948} $\pm$ 0.006 & \textbf{0.875} $\pm$ 0.012    & \textbf{0.939} $\pm$ 0.006   \\ \midrule
\end{tabular}
}
}

\end{table}
\begin{table}[!htp]
\centering
\caption{\textbf{Sample efficiency analysis} on the real-world Mini-Cheetah contact dataset~\citep{DBLP:conf/corl/LinZYG21}. The dataset includes 634.6K training and validation samples. This table presents the legs-averaged F1 scores when training on different proportions of the data. Results correspond to Fig.~\ref{fig:contact-detection-exp}-right.}
\label{tab:sample-efficiency}
{\fontsize{9}{11}\selectfont
\begin{tabular}{@{}lccccccccc@{}}
\toprule
\multirow{2}{*}{\textbf{Model}} & \multirow{2}{*}{\textbf{Sym.}} & \multicolumn{8}{c}{\textbf{Training Samples (\%)}}            \\ \cmidrule(l){3-10} 
                                &                                & 2.50  & 5.00  & 10.00 & 15.00 & 21.25 & 42.50 & 63.75 & 85.00 \\ \midrule
CNN                             & -                              & 0.745 & 0.794 & 0.831 & 0.802 & 0.811 & 0.840 & 0.850 & 0.836 \\
CNN-Aug                         & $\mathbb{C}_2$                 & 0.764 & 0.851 & 0.827 & 0.859 & 0.844 & 0.829 & 0.839 & 0.881 \\
ECNN                            & $\mathbb{C}_2$                 & 0.840 & 0.841 & 0.851 & 0.843 & 0.867 & 0.877 & 0.785 & 0.881 \\
MI-HGNN                         & $\mathbb{S}_4$                 & 0.872 & 0.908 & 0.926 & 0.930 & 0.937 & 0.940 & 0.932 & 0.931 \\
\midrule
MS-HGNN                         & \textcolor[HTML]{8C96C6}{$\mathbb{C}_2$}                 & 0.760 & 0.893 & 0.910 & 0.923 & 0.926 & 0.939 & 0.935 & 0.939 \\
MS-HGNN                         & \textcolor[HTML]{6D016B}{$\mathbb{K}_4$}                 & 0.869 & 0.897 & 0.913 & 0.922 & 0.919 & 0.939 & 0.935 & 0.942 \\ \bottomrule
\end{tabular}
}
\end{table}
\begin{table}[!htp]
\centering
\caption{\textbf{Centroidal momentum estimation results} on the synthetic Solo dataset~\citep{DBLP:conf/rss/ApraezMAM23}. This table reports the numerical results corresponding to Fig.~\ref{fig:com-exp}-left. Evaluation metrics include linear cosine similarity, angular cosine similarity, and test MSE (mean$\pm$std over 4 runs). \textbf{Bold} and \underline{underlined} indicate the best and second-best results, respectively.}
\label{tab:com-exp}
{\fontsize{9}{11}\selectfont
\begin{tabular}{@{}lcccc@{}}
\toprule
\textbf{Model} & \textbf{Sym.}  & \textbf{Lin. Cos. Sim. $\uparrow$} & \textbf{Ang. Cos. Sim. $\uparrow$} & \textbf{Test MSE $\downarrow$}  \\ \midrule
MLP            & -              & 0.9617 $\pm$ 0.0036                & 0.9523 $\pm$ 0.0032                & 0.0499 $\pm$ 0.0037             \\
MLP-Aug        & $\mathbb{C}_2$ & 0.9639 $\pm$ 0.0026                & 0.9535 $\pm$ 0.0029                & 0.0478 $\pm$ 0.0020             \\
MLP-Aug        & $\mathbb{K}_4$ & 0.9647 $\pm$ 0.0023                & 0.9549 $\pm$ 0.0023                & 0.0472 $\pm$ 0.0014             \\
EMLP           & $\mathbb{C}_2$ & 0.9610 $\pm$ 0.0039                & 0.9528 $\pm$ 0.0051                & 0.0503 $\pm$ 0.0053             \\
EMLP           & $\mathbb{K}_4$ & 0.9673 $\pm$ 0.0045                & 0.9580 $\pm$ 0.0032                & 0.0435 $\pm$ 0.0048             \\
MI-HGNN        & $\mathbb{S}_4$ & 0.9301 $\pm$ 0.0017                & 0.5173 $\pm$ 0.0016                & 0.3421 $\pm$ 0.0009             \\
\midrule
MS-HGNN        & \textcolor[HTML]{8C96C6}{$\mathbb{C}_2$} & \textbf{0.9903} $\pm$ 0.0001       & \textbf{0.9804} $\pm$ 0.0015       & \textbf{0.0161} $\pm$ 0.0006    \\
MS-HGNN        & \textcolor[HTML]{6D016B}{$\mathbb{K}_4$} & \underline{0.9877} $\pm$ 0.0007    & \underline{0.9799} $\pm$ 0.0010    & \underline{0.0189} $\pm$ 0.0007 \\ \bottomrule
\end{tabular}
}
\end{table}
\begin{table}[!htp]
    \centering
    \caption{\textbf{Model efficiency comparison between MI-HGNN and MS-HGNN} on the CoM momentum estimation task. This table shows the first part of the numerical results in Fig.~\ref{fig:com-exp}-right, reporting linear cosine similarity with varying parameter counts.}
    \label{tab:com-model-eff-1}
    {\fontsize{9}{11}\selectfont
    \begin{tabular}{@{}cc|cc|cc@{}}
\toprule
\textbf{\# of Param.} & \textbf{MI-HGNN} & \textbf{\# of Param.} & \textbf{MS-HGNN (\textcolor[HTML]{8C96C6}{$\mathbb{C}_2$})} & \textbf{\# of Param.} & \textbf{MS-HGNN (\textcolor[HTML]{6D016B}{$\mathbb{K}_4$})} \\ \midrule
12,934                & 0.8864           & 13,478                & 0.9448                            & 11,366                & 0.9240                            \\
25,478                & 0.9136           & 26,150                & 0.9558                            & 21,926                & 0.9505                            \\
50,438                & 0.9213           & 52,550                & 0.9746                            & 44,230                & 0.9675                            \\
100,102               & 0.9297           & 102,470               & 0.9870                            & 85,830                & 0.9854                            \\
199,174               & 0.9275           & 207,494               & 0.9903                            & 174,470               & 0.9875                            \\
223,878               & 0.9325           & 405,638               & 0.9945                            & 339,590               & 0.9935                            \\
396,806               & 0.9319           & 464,838               & 0.9940                            & 390,726               & 0.9915                            \\
791,558               & 0.9276           & 824,582               & 0.9959                            & 692,998               & 0.9936                            \\ \bottomrule
\end{tabular}
    }
\end{table}

\begin{table}[!htp]
    \centering
    \caption{\textbf{Model efficiency comparison of MLP, MLP-Aug, and EMLP} on the CoM momentum estimation task~\citep{ordoñezapraez2024morphologicalsymmetriesrobotics}. This table presents the second part of the numerical results in Fig.~\ref{fig:com-exp}-right, showing linear cosine similarity across different architectures and symmetry configurations.}
    \label{tab:com-model-eff-2}
    {\fontsize{9}{11}\selectfont
\begin{tabular}{@{}cccc|ccc@{}}
\toprule
\textbf{\# of Param.} & \textbf{MLP} & \textbf{MLP-Aug ($\mathbb{C}_2$)} & \textbf{MLP-Aug ($\mathbb{K}_4$)} & \textbf{\# of Param.} & \textbf{EMLP ($\mathbb{C}_2$)} & \textbf{EMLP ($\mathbb{K}_4$)} \\ \midrule
10,310                & 0.9147       & 0.9170                            & 0.9199                            & -                     & -                              & -                              \\
36,998                & 0.9631       & 0.9660                            & 0.9644                            & 36,992                & 0.9640                         & 0.9718                         \\
139,526               & 0.9737       & 0.9802                            & -                                 & 139,520               & 0.9843                         & 0.9868                         \\
541,190               & 0.9643       & 0.9839                            & 0.9843                            & 541,184               & 0.9865                         & 0.9910                         \\ \bottomrule
\end{tabular}
}
\end{table}
\enlargethispage{-\baselineskip}
\end{document}